\documentclass{article}
\usepackage[utf8x]{inputenc}

\usepackage{url}
\usepackage{microtype}      %
\usepackage{comment}        %
\usepackage{amsmath}
\usepackage{multirow}
\usepackage{bm}
\usepackage{amsthm}
\usepackage{amssymb}
\usepackage{makecell}
\usepackage{mathtools}
\usepackage{color}
\usepackage{xcolor}
\usepackage{makecell}
\usepackage{graphicx}
\usepackage{arydshln}
\usepackage{booktabs}
\usepackage{framed}
\usepackage{caption}
\usepackage[scientific-notation=true]{siunitx}
\usepackage{wrapfig}
\usepackage{algorithm}
\usepackage{algorithmic}
\usepackage{lipsum}
\usepackage{sidecap}
\usepackage{pifont}%
\usepackage[stable]{footmisc}

\usepackage{enumitem}
\setlist[enumerate]{topsep=0pt,itemsep=-1ex,partopsep=1ex,parsep=1ex}
\setlist[itemize]{topsep=0pt,itemsep=-1ex,partopsep=1ex,parsep=1ex}

\definecolor{gg}{RGB}{15,150,15}
\definecolor{rr}{RGB}{230,45,45}

\makeatletter
\def\maketag@@@#1{\hbox{\m@th\normalfont\normalsize#1}}
\makeatother

\usepackage{amsmath,amsfonts,bm}
\usepackage{multirow}
\usepackage{tabularx}
\usepackage{graphicx}
\usepackage{adjustbox}
\usepackage{amsthm}
\usepackage{xspace}

\newcommand{\mset}[1]{\left\{\kern-.5em\left\{ #1 \right\}\kern-.5em\right\}}
\newcommand{\mmset}[1]{\{\kern-.4em\{ #1 \}\kern-.4em\}}

\newcommand{\norm}[1]{\left\Vert#1\right\Vert}

\newcommand{\abs}[1]{\left\vert#1\right\vert}

\newcommand{\set}[1]{\left\{#1\right\}}
\newcommand{\parr}[1]{\left (#1\right )}
\newcommand{\brac}[1]{\left [#1\right ]}

\newcommand{\Real}{\mathbb R}

\newcommand{\eps}{\varepsilon}
\newcommand{\too}{\rightarrow}
\newcommand{\dtoo}{\downarrow}

 \newtheorem{theorem}{Theorem}

 \newtheorem{definition}{Definition}

\def\eqref#1{equation~\ref{#1}}

\def\1{\bm{1}}

\def\eps{{\epsilon}}

\def\vec1{{\bm{1}}}

\DeclareMathAlphabet{\mathsfit}{\encodingdefault}{\sfdefault}{m}{sl}
\SetMathAlphabet{\mathsfit}{bold}{\encodingdefault}{\sfdefault}{bx}{n}

\def\gC{{\mathcal{C}}}

\def\gM{{\mathcal{M}}}

\def\gP{{\mathcal{P}}}

\def\gS{{\mathcal{S}}}
\def\gT{{\mathcal{T}}}
\def\gU{{\mathcal{U}}}

\newcommand{\E}{\mathbb{E}}

\newcommand{\R}{\mathbb{R}}

\DeclareMathOperator*{\argmin}{arg\,min}

\usepackage[accepted,nohyperref]{icml2021}

\usepackage{xcolor}
\definecolor{linkcolor}{RGB}{74, 102, 146}

\usepackage[colorlinks=true,allcolors=linkcolor,pageanchor=true,plainpages=false,pdfpagelabels,bookmarks,bookmarksnumbered]{hyperref}

\makeatletter
\DeclareRobustCommand\onedot{\futurelet\@let@token\@onedot}
\def\@onedot{\ifx\@let@token.\else.\null\fi\xspace}
\newcommand{\eg}{\emph{e.g}\onedot}
\newcommand{\ie}{\emph{i.e}\onedot}
\makeatother

\usepackage[nameinlink]{cleveref}
\Crefname{section}{Sect.}{Sects.}
\Crefname{appendix}{App.}{Apps.}
\Crefname{proposition}{Prop.}{Props.}

\AtBeginDocument{%
  \addtolength\abovedisplayskip{-0.3\baselineskip}%
  \addtolength\belowdisplayskip{-0.3\baselineskip}%
}

\icmltitlerunning{Riemannian Convex Potential Maps}

\begin{document}

\twocolumn[
\icmltitle{Riemannian Convex Potential Maps}

\icmlsetsymbol{equal}{*}

\begin{icmlauthorlist}
\icmlauthor{Samuel Cohen}{equal,ucl}
\icmlauthor{Brandon Amos}{equal,fb}
\icmlauthor{Yaron Lipman}{fb,weizmann}
\end{icmlauthorlist}

\icmlaffiliation{fb}{Facebook AI Research}
\icmlaffiliation{ucl}{University College London}
\icmlaffiliation{weizmann}{Weizmann Institute of Science}
\icmlcorrespondingauthor{Samuel Cohen}{samuel.cohen.19@ucl.ac.uk}
\icmlcorrespondingauthor{Brandon Amos}{brandon.amos.cs@gmail.com}
\icmlcorrespondingauthor{Yaron Lipman}{ylipman@fb.com,yaron.lipman@weizmann.ac.il}

\icmlkeywords{optimal transport, generative modeling,
  Riemannian geometry, convex optimization}

\vskip 0.3in
]

\printAffiliationsAndNotice{\icmlEqualContribution}

\begin{abstract}
  Modeling distributions on Riemannian manifolds is a
  crucial component in understanding non-Euclidean data that
  arises, \eg, in physics and geology.
  The budding approaches in this space are limited by
  representational and computational tradeoffs.
  We propose and study a class of flows that uses
  convex potentials from Riemannian optimal transport.
  These  are universal and can model distributions on
  any compact Riemannian manifold without requiring domain knowledge
  of the manifold to be integrated into the architecture.
  We demonstrate that these flows can model standard
  distributions on spheres, and tori, on synthetic and geological data.
  Our source code is freely available online at
  \href{http://github.com/facebookresearch/rcpm}{github.com/facebookresearch/rcpm}.
\end{abstract}

\section{Introduction}
Today's generative models have had wide-ranging successes
of modeling non-trivial probability distributions that
naturally arise in fields such as
physics \citep{Khler2019EquivariantFS,rezendehamflows},
climate science \cite{mathieu2020riemannian},
and reinforcement learning \cite{flowrl}.
Generative modeling on ``straight'' spaces (\ie, Euclidean) are pretty
well-developed and include (continuous) normalizing flows
\citep{varinfflows, dinh2016density, chenneuralodes},
generative adversarial networks
\citep{gangoodfellow},
and variational auto-encoders
\citep{vaekingma,rezendevae}.

In many applications however, data resides on spaces with more
complicated structure, \eg, Riemannian manifolds such as
spheres, tori, and cylinders. Using Euclidean generative models on
this data is problematic from two aspects: first, Euclidean
models will allocate mass in ``infeasible'' areas of the space; and
second, Euclidean models will often need to squeeze mass in zero
volume subspaces. Moreover, knowledge of the space geometry can improve the
learning process by incorporating an efficient geometric inductive bias
as part of the modeling and learning pipeline.

Flow-based generative models are the state-of-the-art in
Euclidean settings and are starting to be extended
to Riemannian manifolds \citep{rezende2020normalizing,mathieu2020riemannian,lou2020neural}.
However, in contrast with some models in the Euclidean case
\cite{expressivepowflows, huang2020convex}, the representational
capacity and universality of these models is not well-understood.
Some of these approaches are efficiently tailored to specific
choices of manifolds, but the methods and theory of flows on
general Riemannian manifolds are not well-understood.

In this paper we introduce the Riemannian Convex Potential Map (RCPM),
a generic model for generative modeling on arbitrary
Riemannian manifolds that enjoys universal representational power.
RCPM (illustrated in \cref{fig:demo}) is based on Optimal Transport (OT) over Riemannian
manifolds \cite{mccann2001polar, villani2008optimal, sei2013jacobian, rezende2020normalizing}
and generalizes the convex potential flows in the Euclidean
setting by \citet{huang2020convex}.
We prove that RCPMs are universal  on \emph{any} compact
Riemannian manifold, which comes from the fact that our discrete
$c$-concave potential functions are universal.
Our experimental demonstrations show that RCPMs are competitive
and model standard distributions on spheres and tori.
We further show a case study in modeling continental drift where
we transport Earth's land mass on the sphere.

\begin{figure}
    \centering
    \includegraphics[width=0.995\columnwidth]{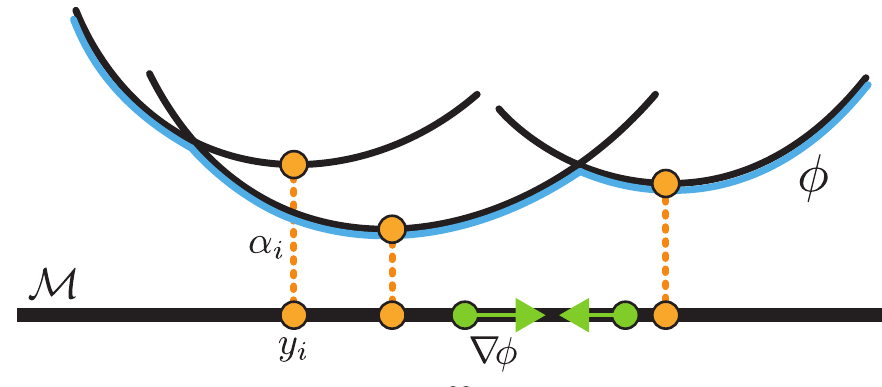}
    \caption{Illustration of a discrete $c$-concave
      function $\phi$ (blue) over a base manifold $\gM$ (bold line).
      These consist of discrete components $\{\alpha_i, y_i\}$ and
      have a Riemannian gradient $\nabla \phi\in T_x\gM$.
    }
    \label{fig:discrete_c_concave}
\end{figure}

\begin{figure*}[t]
    \centering
    \hspace*{-40mm}
    \includegraphics[width=0.248\textwidth]{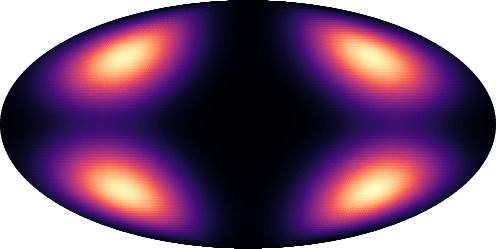}
    \hspace*{-1.75mm}
    \includegraphics[width=0.248\textwidth]{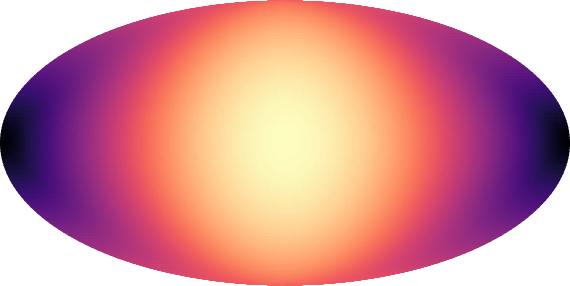}
    \hspace*{-1.5mm}
    \includegraphics[width=0.24\textwidth]{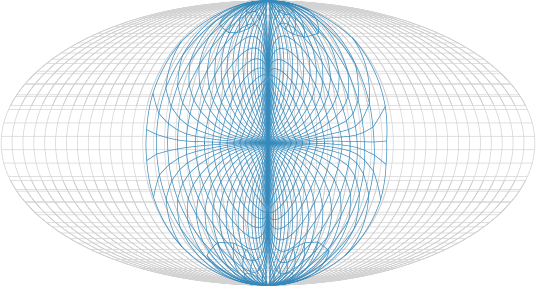}
    \hspace*{-1.5mm}
    \includegraphics[width=0.248\textwidth]{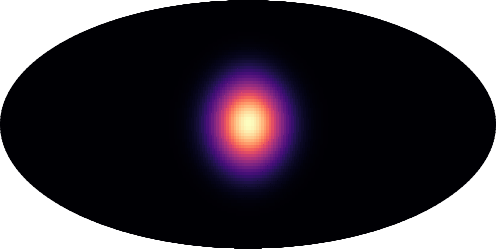}
    \hspace*{-40mm}
    \caption{Illustration of a Riemannian convex potential map on a sphere.
      From left to right:
      1) base distribution $\mu$ of a mixture of wrapped Gaussians,
      2) learned $c$-convex potential,
      3) mesh grid distorted by the exponential map of the
         Riemannian gradient of the potential,
      4) transformed distribution $\nu$.
    }
    \label{fig:demo}
\end{figure*}

\section{Related Work}
\paragraph{Euclidean potential flows.}
Most related to our work, is the work by \citet{huang2020convex} that
leveraged Euclidean optimal transport, parameterized using input
convex neural networks (ICNNs) \cite{amos2017input} to construct universal
normalizing flows on Euclidean spaces. Similarly, \citet{w2gn,oticnn} compute optimal transport maps via ICNNs. Riemannian optimal transport replaces
the standard Euclidean convex functions with so-called $c$-convex or
$c$-concave functions, and the Euclidean translation by exponential
map. Unfortunately, the notion of $c$-convex or $c$-concave functions
is intricate and a simple characterization of such functions is not
known.
Our approach is to approximate arbitrary $c$-concave functions on
general Riemannian manifolds using \emph{discrete $c$-concave
functions} that are simply the minimum of a finite number of
translated squared intrinsic distance functions, see
\cref{fig:discrete_c_concave}. Intuitively, this construction resembles
the approximation of a Euclidean concave function as the minimum of a
finite collection of affine tangents. Although simple, we prove that
\emph{discrete $c$-concave functions} are in fact \emph{dense} in the
space of $c$-concave functions and therefore replacing general
$c$-concave functions with discrete $c$-concave functions leads to a
universal Riemannian OT model. Related, \citet{gangbo1996} considered OT maps
of discrete measures which are defined via discrete $c$-concave functions. %

\paragraph{Exponential map flows.}
\citet{sei2013jacobian, rezende2020normalizing} propose distinct
parameterizations for $c$-convex functions living on the sphere
specifically. The latter applies it to training flows on the sphere
using the construction from McCann's theorem.  Our work can be seen as
a generalization of the exponential-map approach in
\citet{rezende2020normalizing} to arbitrary Riemannian
manifolds. In contrast to this work,
the maps from
our discrete $c$-concave layers are universal.

\paragraph{Other Riemannian flows.}
\citet{mathieu2020riemannian,lou2020neural} propose extensions of
continuous normalizing flows to the Riemannian manifold setting. These
are flexible with respect to the choice of manifold, but their
representational capacity is not well-understood and solving ODEs on
manifolds can be expensive.
In parallel, \citet{brehmer2020flows} proposed a method for
simultaneously learning the manifold data lives on and a normalizing
flow on the learned
manifold.
\citet{bose2020latent} consider hyperbolic normalizing flows.

\paragraph{Optimal transport on Riemannian manifolds.}
Optimal transport on spherical manifolds has been extensively studied from theoretical standpoints. \citet{figalli2009continuity, regularityotsp,kim2012towards} study the regularity (continuity, smoothness) of transport maps on spheres and other non-negatively curved manifolds. Regularity and smoothness are more intricate on negatively curved manifolds, \eg hyperbolic spaces. Nevertheless, several works demonstrated that transport can be made smooth through a minor change to the Riemannian cost  \citep{smoothhyp}.
\citet{othypalv, aligninghyp} leverage this to learn transport
maps on hyperbolic spaces, in which case maps are parameterized as
hyperbolic neural networks.  %

\section{Background}
In this section, we introduce the relevant background on normalizing flows and Riemannian optimal transport theory.

\subsection{Normalizing flows}
Normalizing flows parameterize probability distributions $\nu\in\gP(\gM)$, on a manifold $\gM$, by pushing a simple base (prior) distribution $\mu \in \gP(\gM)$ through a diffeomorphism\footnote{A diffeomorphism is a differentiable bijective mapping with a differentiable inverse.} $s:\gM\too\gM$.

In turn, sampling from distribution $\nu$ amounts to transforming samples $x$ taken from the base distribution via $s$:
\begin{align}
   y = s(x) \sim \nu, \quad \text{where } x \sim \mu.
\end{align}
In the language of measures, $\nu$ is the push-forward of the base measure $\mu$ through the transformation $s$, denoted by $\nu = s_{\#}\mu$. If densities exist, then they adhere the change of variables formula
\begin{align}
  \nu(y) = \mu(x)| \det J_s(x)|^{-1},
\end{align}
where we slightly abuse notation by denoting the densities again as $\mu,\nu$.
In practice, a normalizing flow $s$ is often defined as a composition of simpler, primitive diffeomorphisms $s_1,\ldots,s_T:\gM\too\gM$, \ie,
\begin{equation}\label{e:s_as_comp}
    s=s_T\circ\cdots\circ s_1.
\end{equation}
For a more substantial review of computational and representational trade-offs inherent to this class of model on Euclidean spaces, we refer to \citet{flowreview}.

\subsection{$c$-convexity and concavity}
\label{ss:c-convexity}
Let $(\gM, g)$ be a smooth compact Riemannian manifold without boundary, and $c(x,y)=\frac{1}{2}d(x,y)^2$, where $d(x,y)$ is the intrinsic distance function on the manifold. We use the following generalizations of convex and concave functions:
\begin{definition}
  A function $\phi:\gM\too \Real\cup\set{+\infty}$ is $c$-convex if it is not identically $+\infty$ and there exists $\psi:\gM\too\Real\cup\set{\pm \infty}$ such that
  \begin{equation}\label{e:c_convex}
      \phi(x)=\sup_{y\in\gM} \parr{-c(x,y)+\psi(y)}
  \end{equation}
\end{definition}
\begin{definition}\label{def:c_concave}
  A function $\phi:\gM\too \Real\cup\set{-\infty}$ is $c$-concave if it is not identically $-\infty$ and there exists $\psi:\gM\too\Real\cup\set{\pm \infty}$ such that
  \begin{equation}\label{e:c_concave}
      \phi(x)=\inf_{y\in\gM} \parr{c(x,y)+\psi(y)}
  \end{equation}
\end{definition}
We denote the space of $c$-concave functions on $\gM$ as $\widehat{\gC} (\gM)$.  We also note that if $\psi$ is $c$-concave, $-\psi$ is $c$-convex, hence c-concavity results can be directly extended into c-convexity results by negation.
We also use the $c$-infimal convolution:
\begin{equation}\label{e:c_infimal}
  \psi^c(y) = \inf_{x\in\gM} \parr{ c(x,y) - \psi(x)}.
\end{equation}
$c$-concave functions $\phi$ satisfy the involution property:
\begin{equation}\label{e:involution}
    \phi^{cc}=\phi.
\end{equation}

When $\gM$ is a product of spheres or a Euclidean space, (\eg, spheres, tori),  $\widehat{\gC}(\gM)$ is a convex space \cite{regprodsphere, otcurvature} where a convex combinations of $c$-concave functions are $c$-concave. In the case $\gM = \Real^d$ and $c(x,y) = -x^Ty$, Euclidean concavity is recovered.

\subsection{Riemannian Optimal Transport}
\label{ss:mccan}
Optimal transport deals with finding efficient ways to push a base probability measure $\mu\in \gP(\gM)$ to a target measure $\nu\in\gP(\gM)$, \ie, $s_\#\mu=\nu$. Often $s$ considered is more general than a diffeormorphism, namely a transport plan which is a bi-measure on $\gM\times \gM$.

When $\gM$ is a smooth compact manifold with no boundary, $\mu,\nu\in \gP(\gM)$, and $\mu$ has density (\ie, is absolutely continuous w.r.t.~the volume measure of $\gM$), Theorem 9 of \citet{mccann2001polar} shows that there is a unique (up-to $\mu$-zero sets) transport map $t:\gM\too\gM$ that pushes $\mu$ to $\nu$, \ie, $s_\#\mu=\nu$, and minimizes the transport cost
\begin{align}C(s)=\int_\gM c(x,s(x)) d\mu(x).\end{align} Furthermore, this OT map is given by
\begin{equation}\label{e:ot}
    t(x)=\exp_x \brac{-\nabla \phi(x)},
\end{equation}
where $\phi$ is a $c$-concave function, $\exp$ is the Riemannian exponential map, and $\nabla$ is the Riemannian gradient. Note that equivalently  $t(x)=\exp_x\brac{\nabla\psi(x)}$ for $c$-convex $\psi$.

As a consequence, there always exists a (Borel) mapping $t:\gM \rightarrow \gM$ such that $t_{\#}\mu = \nu$ where $t$ is of the form of \cref{e:ot}. The issue of regularity and smoothness of OT maps is a delicate one and has been  extensively studied (see, \eg, \citet{villani2008optimal, regprodsphere, otcurvature}); in general, OT maps are not smooth, but can be seen as a natural generalization to normalizing flows, relaxing the smoothness of $s$. Henceforth, we will call OT maps ``flows.'' In fact, our discrete $c$-concave functions, the gradient of which are shown to approximate general OT maps, define piecewise smooth maps. 

Constant-speed geodesics $\eta: [0,1] \rightarrow \gM$ between a sample $x$ (from $\mu$) and $t(x)$ can also be recovered $\mu$-almost everywhere on the manifold \citep{otcurvature} as
\begin{equation}\label{e:otgeodesics}
    \eta(l)=\exp_x \brac{-l\nabla \phi(x)}.
\end{equation} 
For a geodesic starting at $x_0 \in \gM$, $\eta(0) = x_0$ and $\eta(1) = t(x_0)$.

\section{Riemannian Convex Potential Maps\footnote{
    As mentioned in \cref{ss:c-convexity,ss:mccan},
    both $c$-convex and $c$-concave can
    be used; we follow \citet{mccann2001polar}
    and use $c$-concavity in the theory
    and derivations here.}}
Our goal is to represent optimal transport maps on
Riemannian manifolds $t:\gM\too \gM$.
The key idea is to build upon the theory of
\citet{mccann2001polar} and parameterize the space of optimal
transport maps by $c$-concave functions $\phi:\gM\too\gM$, see
\cref{def:c_concave}. Given a $c$-concave function, the map $t$ is
computed via \cref{e:ot}. This requires computing the intrinsic
gradient of $\phi$ and the exponential map on $\gM$.

\subsection{Discrete $c$-concave functions}
Let $\set{y_i}_{i\in
[m]}\subset \gM$ be a set of $m$ discrete points, where $[m]=\set{1,2,\ldots,m}$, and define the function $\psi$ to be
\begin{equation}
  \begin{aligned}
  \psi(x) =
    \begin{cases} \alpha_i & \text{if }x=y_i \\
    +\infty & \text{otherwise}
    \end{cases}
  \end{aligned}
\end{equation}
where $\alpha_i\in \Real$ are arbitrary. Plugging this choice in \cref{def:c_concave} of $c$-concave functions, we get that
\begin{equation}\label{e:inf_aff}
    \phi(x)=\min_{i\in [m]} \parr{ c(x,y_i) + \alpha_i }
\end{equation}
is $c$-concave. We denote the collection of these functions over $\gM$ by $\widehat{\gC}^d(\gM)$. We will use this modeling metaphor for parameterizing $c$-concave functions. Our learnable parameters of a single $c$-concave function $\psi$ will consist of
\begin{equation}\label{e:params}
 \theta=\set{(y_i,\alpha_i)}_{i\in [m]}\subset \gM\times \Real.
\end{equation}
Let $i_\star=\argmin_{i\in[m]}\parr{c(x,y_{i})+\alpha_{i}}$. The discrete $c$-concave function in \cref{e:inf_aff} is differentiable, except where two pieces $c(x,y_i)+\alpha_i$ meet, and if $x$ belongs to the cut locus of $y_{i_\star}$ on $\gM$,  which is of volume measure zero \citep{sakaicutlocus}. Excluding such cases, the gradient of $\phi$ at $x$ is:
\begin{equation}
\begin{aligned}
    \nabla_x \phi(x) &= \nabla_x \big[ c(x, y_{i_\star}) + \alpha_\star \big]\\
    &=\nabla_x  c(x, y_{i_\star}) = -\log_x(y_{i_\star}),
\end{aligned}
\label{e:grad_pot}
\end{equation}
where $\log$ is the logarithmic map on the manifold. See \cref{fig:discrete_c_concave} for an illustration of a discrete $c$-concave function.
Intuitively, the optimal transport generated by discrete $c$-concave functions is piecewise constant as $\exp_x(- \nabla_x \phi(x)) = \exp_x(-(-\log_x(y_{i_\star}))) = y_{i_\star}$. This can be seen as generalizing semi-discrete optimal transport  \cite{compoptpeyre}, which aims at finding transport maps between continuous and discrete probability measures, to the manifold setting.

\paragraph{Relation to the Euclidean concave case.}
In the Euclidean setting, \ie, when $\gM=\Real^d$ and $c(x,y)=-x^Ty$, a Euclidean concave (closed) function $\phi$ can be expressed as $$\phi(x) = \inf_{y\in\Real^d} \parr{ -x^Ty + \psi(y)}.$$
Replacing $\Real^d$ with a finite set of points $y_i\in\Real^d$, $i\in[m]$, leads to the \emph{discrete Legendre-Fenchel transform} \cite{lucet1997faster}; it basically amounts to approximating the concave function $\phi$ via the minimum of a collection of affine functions. This transform can be shown to converge to $\phi$ under refinement \cite{lucet1997faster}. We next prove convergence of discrete $c$-concave functions to their continuous counterparts.

\paragraph{Expressive power of discrete $c$-concave functions.}
Let us show that \cref{e:inf_aff} can approximate \emph{arbitrary} $c$-concave functions $\phi:\gM\too\Real\cup \set{\infty}$ on compact manifolds $\gM$.
We will prove the following theorem:
\begin{theorem}\label{thm:discrete_expressive}
For compact, boundaryless, smooth manifold $\gM$, we have $\widehat{\gC}^d(\gM)$ dense in $\widehat{\gC}(\gM)$.
\end{theorem}
By dense we mean that for every $\hat{\phi}\in\widehat{\gC}(\gM)$ there exists a sequence  $\phi_\eps\in\widehat{\gC}^d(\gM)$, where $\eps\dtoo 0$, so that for almost all $x\in\gM$ we have that $\phi_\eps(x)\too \hat{\phi}(x)$ and $\nabla_x\phi_\eps(x)\too \nabla_x\hat{\phi}(x)$, as $\eps\dtoo 0$.

The proof is based on a construction of $\phi_\eps$ using an $\eps$-\emph{net} of $\gM$. A set of points $\set{y_i}_{i\in [m]}\subset \gM$ is called $\eps$-net if $\gM\subset \cup_{i\in[m]} B(y_i,\eps)$, where $B(y_i,\eps)$ is the $\eps$-radius ball centered at $y_i$. Formulated differently, every point $y\in\gM$ has a point in the net that is at-most $\eps$ distance away. On compact manifolds, for arbitrary $\eps>0$, there exists a finite $\eps$-net $\set{y_i}_{i\in [m]}$. Note that $m\too\infty$ as $\eps\dtoo 0$, but it is finite for every particular $\eps$. Our candidate for approximating $\hat{\phi}$ is:
\begin{equation}\label{e:phi_eps}
     \phi_\eps(x) = \min_{i\in [m]} \parr{ c(x,y_i) -   \hat{\phi}^c(y_i) },\\
\end{equation}
where $\hat{\phi}^c$ is the infimal $c$-convolution (see \cref{e:c_infimal}) of $\hat{\phi}$. The approximation in \cref{e:phi_eps} is motivated by the involution property (\cref{e:involution}). $\hat{\phi}=(\hat{\phi}^c)^c$ and therefore
$$\hat{\phi}(x) = \inf_{y\in\gM} \parr{ c(x,y) - \hat{\phi}^c(y)}.$$

\begin{proof}
Let $\hat{\phi}:\gM\too\Real$ be an arbitrary $c$-concave function over $\gM$. Let $\eps\dtoo 0$ denote a sequence of positive numbers converging monotonically to zero. We will show that $\phi_\eps$ defined in \cref{e:phi_eps} converges uniformly to $\hat{\phi}$ over $\gM$ and furthermore, that their Riemannian gradients $\nabla \phi_\eps(x)$ converge pointwise to $\nabla \hat{\phi}(x)$ for almost all $x\in \gM$ (\ie, up to a set of zero volume).

\paragraph{Uniform convergence.}
We start by noting that $\hat{\phi}^c$ is also $c$-concave by definition, and Lemma 2 in \citet{mccann2001polar} implies that $\hat{\phi}^c$ is $|\gM|$-Lipschitz, namely
$$\abs{\hat{\phi}^c(x)-\hat{\phi}^c(y)}\leq |\gM| d(x,y),$$
for all $x,y\in \gM$.  We denote by $|\gM|$ the diameter of $\gM$, that is:
\begin{equation}
    \abs{\gM} = \sup_{x,y\in\gM} d(x,y),
\end{equation}
and $|\gM|<\infty$ since $\gM$ is compact. In particular $\hat{\phi}^c$ is either everywhere infinite (non-interesting case), or is finite (in fact, bounded) over $\gM$.

Next, we establish an upper bound. For all $x\in \gM$:
\begin{equation}
\begin{aligned}
  \hat{\phi}(x) &= \inf_{y\in\gM} \parr{c(x,y)-\hat{\phi}^c(y)} \\ &\leq  \min_{i\in [m]} \parr{ c(x,y_i) -   \hat{\phi}^c(y_i) }\\ &= \phi_\eps(x).
\end{aligned}
\end{equation}
Note that this upper bound is true for all choices of $y_i$. Next, we show a tight lower bound. 

Furthermore, Lemma 1 in \citet{mccann2001polar} asserts that $c(x,y)=\frac{1}{2}d(x,y)^2$ is also $|\gM|$-Lipschitz as a function of each of its variables. Therefore, using the $\eps$-net, we get that for each $x,y\in\gM$ there exists $i\in [m]$ so that
$$c(x,y)-\hat{\phi}^c(y) \geq c(x,y_i) - \hat{\phi}^c(y_i) - 2|\gM|\eps$$
leading to
\begin{equation}
\label{e:temp}
\begin{aligned}
  \hat{\phi}(x) &= \inf_{y\in\gM} \parr{ c(x,y) - \hat{\phi}^c(y)} \\ &\geq \min_{i\in [m]} \parr{ c(x,y_i) + \hat{\phi}^c(y_i) } - 2|\gM|\eps \\ &= \phi_\eps(x) - 2|\gM|\eps
\end{aligned}
\end{equation}
Therefore we have that $\phi_\eps$ converge uniformly in $\gM$ to $\hat{\phi}$.

\paragraph{Pointwise convergence of gradients.}
Let $O\subset \gM$ be the set of points where the gradients of $\hat{\phi}$ and $\phi_\eps$ (for the entire countable sequence $\eps$) are not defined, then $O$ is of volume-measure zero on $\gM$. Indeed, the functions $\hat{\phi}, \phi_\eps$ are differentiable almost everywhere on $\gM$ by Lemmas 2 and 4 in \citet{mccann2001polar}.
Furthermore, if we denote by $\hat{t}$ the optimal transport defined by $\hat{\phi}$, as discussed in Chapter 13 in \citet{villani2008optimal} the set of all $x\in\gM$ for which $\hat{t}(x)$ belongs to the cut locus is of measure zero. We add this set to $O$, keeping it of measure zero.

Fix $x\in \gM \setminus O$, and choose an arbitrary $\rho>0$. 
We show convergence of $\nabla_x \phi_\eps(x) \too \nabla_x \hat{\phi}(x)$ by showing we can take element $\eps$ small enough so that the two tangent vectors $\nabla_x \phi_\eps(x), \nabla_x \hat{\phi}(x)\in T_x\gM$ are at most $\rho$ apart.

Lemma 7 in \citet{mccann2001polar} shows that the unique minimizer of
$h(y) = c(x,y)-\hat{\phi}^c(y)$ is achieved at $y_*=\exp_x[-\nabla_x\hat{\phi}(x)]$. In particular,  $\nabla_x\hat{\phi}(x) = -\log_x(y_\star)$. As explained above, $y_*$ is not on the cut locus of $x$.

Recall that $\nabla_x c(x,y) = -\log_x(y)$ \citep{mccann2001polar}, which is a continuous function of $y$ in vicinity of $y_*$. Therefore there exists an $\eps'>0$ so that if $y\in B(y_*,\eps')$ we have that $\norm{-\log_x(y)+\log_x(y_*)}<\rho$, where the norm is the Riemannian norm in the tangent space at $x$, \ie, $T_x\gM$.

Consider the set $A_\delta = \set{y\in\gM \ \vert\ h(y) < h(y_*)+\delta}$. Since $h(y)$ is continuous (in fact, Lipschitz) and $y_*$ is its unique minimum, we can find a $0<\delta$ sufficiently small  so that $A_\delta\subset B(y_*,\eps')$. This means that any $y\notin B(y_*,\eps')$ satisfies $h(y)\geq h(y_*)+\delta$.
On the other hand, from continuity of $h$ we can find $\eps<\eps'$ so that all $y\in B(y_*,\eps)$ we have $h(y) < h(y_*)+\delta$.

Now consider the element $\phi_{\eps}$. Due to the $\eps$-net we know there is at-least one $y_i\in B(y_*,\eps)$ leading to $h(y_i)<h(y_*)+\delta$, and as mentioned above every $y\notin B(y_*,\eps')$ satisfies $h(y)\geq h(y_*)+\delta$. This means that the $y_i$ that achieves the minimum of $h(y_i)$ among all $i\in[m]$ in \cref{e:phi_eps} has to reside in $B(y_*,\eps')$, and $\phi_{\eps}(x)=c(x,y_i)-\hat{\phi}^c(y_i)$ in a small neighborhood of $x$.
Therefore, $\nabla_x \phi_{\eps}(x) = -\log_x(y_i)$.
Since $\nabla_x \hat{\phi}(x) = -\log_x(y_*)$ and $d(y_i,y_*)<\eps'$, our choice of $\eps'$ implies that  $\norm{\nabla_x \hat{\phi}(x)-\nabla_x \phi_{\eps}(x)}<\rho$.
\end{proof}

\subsection{RCPM  architecture}
\label{sec:RCPM_block}

Now that we have set-up an expressive approximation to $c$-concave functions we can take the same route as  \citet{rezende2020normalizing}, and define individual flow
blocks $s_j$, $j\in [T]$ (see \cref{e:s_as_comp}) using the exponential map as suggested by McCann's theorem.
Each flow block $s_j$ is defined as:
\begin{align}\label{e:flow_layers}
    s_j(x)&=\exp(-\nabla_x\phi_j(x)), \quad j=1,\ldots, T\\
    \label{e:flow_layers2}
    \phi_j(x) &= \min_{i\in [m]} \parr{ c(x,y^{(j)}_i) + \alpha^{(j)}_i}.
\end{align}
We learn both locations $y_i^{(j)} \in \gM$ and offsets
$\alpha_i^{(j)} \in \R$ for $i\in [m]$ and $j\in[T]$; these form our model parameters $\theta$. We also consider multi-layer blocks as detailed later.

\subsection{Universality of RCPM}
We next build upon \cref{thm:discrete_expressive} to show RCPM is universal. We show that a single block $s$, \ie, \cref{e:flow_layers,e:flow_layers2} with $T=1$ can already approximate arbitrary the optimal transport $t:\gM\too\gM$. Due to the theory of  \citet{mccann2001polar} (see \cref{ss:mccan}) this means that $s$ can push any absolutely continuous base probability $\mu$ to a general $\nu$ arbitrarily well.
\begin{theorem}
\label{thm:universalRCPM}
If $\mu$, $\nu$ are two probability measures in $\gP(\gM)$ and $\mu$ is absolutely continuous w.r.t volume measure of $\gM$, then there exists a sequence of discrete $c$-concave potentials $\phi_{\epsilon}$, where $\eps\dtoo 0$, such that $\exp\brac{-\nabla \phi_{\epsilon}} \xrightarrow[]{p} t$, where $t$ is the optimal map pushing $\mu$ to $\nu$ and $p$ denotes convergence in probability.
\end{theorem}

\begin{proof}
Let $\phi_\eps$ be the sequence from \cref{e:phi_eps}. It is enough to show pointwise convergence of $\exp[-\nabla\phi_\eps(x)]$ to $t(x)=\exp[-\nabla\hat{\phi}(x)]$ for $\mu$-almost every $x$. Note, as above, that the set of points $O\subset \gM$ where the gradients of $\phi$ and $\phi_\eps$ are not defined is of $\mu$-measure zero. So fix $x\in \gM \setminus O$. 

\Cref{thm:discrete_expressive} implies that the tangent vector $\nabla_x \phi_\eps(x) \in T_x\gM$ converges in the Riemannian norm over $T_x\gM$ to $\nabla_x \phi(x) \in T_x\gM$. Furthermore, from the Hopf-Rinow Theorem $\exp$ is defined over all $T_x\gM$ and it is continuous where it is defined \cite{mccann2001polar}. This shows the pointwise convergence. \end{proof}

As a result of \cref{thm:universalRCPM}, the multi-block version of RCPM is also universal, because individual blocks can approximate the identity arbitrarily well according to \cref{thm:discrete_expressive}.

\section{On Implementing and Training RCPMs}

We now describe how to train Riemannian convex potential maps and how to increase their flexibility and expressivity through architectural choices preserving $c$-concavity. 

\subsection{Variants of RCPM}
\label{sec:multilayer}
Our basic model is multi-block $s=s_T\circ\cdots\circ s_1$, where $s_i$ are defined in \cref{e:flow_layers,e:flow_layers2}. We also consider two variants of this model. Let us denote $\sigma(s)=\min\set{0,s}$, the concave analog of ReLU.

\paragraph{Multi-layer block on convex spaces.}
First, in some manifolds, $c$-concave functions form a convex space, that is convex combination of $c$-concave functions is again $c$-concave. Examples of such spaces include  Euclidean spaces, spheres \cite{sei2013jacobian}, and product of spheres (\eg, tori) \citep{multiplespheresfig}. One possibility to enrich our discrete $c$-concave model in such spaces is to convex combine and compose multiple $c$-concave potentials which preserves $c$-concavity, similar in spirit to ICNN \cite{amos2017input}. In more detail, we define the $c$-concave potential of a single block $\varphi_j$, $j\in [T]$, to be a convex combination and composition of several discrete $c$-concave functions. For brevity let $\varphi=\varphi_j$, and define $\varphi=\psi_K$, where $\psi_K$ is defined by:
\begin{equation}
\begin{aligned}
    \psi_0 &= 0 , \\
    \psi_{k} &= (1-w_{k-1})\phi_{k-1} + w_{k-1} \sigma( \psi_{k-1}),
\end{aligned}
\end{equation}
where $k\in [K]$, $w_k\in [0,1]$ are learnable weights, and $\phi_k\in \widehat{\gC}^d(\gM)$ are discrete $c$-concave functions used to define the $j$-th block. The RCPMs from \cref{sec:RCPM_block} can be reproduced with $K=1$. In general RCPMs in this case are composed by $T$ blocks, each is built out of $K$ discrete $c$-concave function.

\paragraph{Identity initialization.}
In the general case (\ie, even in manifolds where $c$-concave functions are not a convex space) one can still define 
\begin{equation}\label{e:identity_reproduction}
    \varphi_j(x) = \sigma (\phi_j(x)).
\end{equation}
We note that if all $\alpha_i\geq 0$ at initialization, $\sigma (\phi_j(x))\equiv 0$. In this case, we claim that the initial flow is the identity map, that is $s(x)=x$. Indeed, the gradient of a constant function vanishes everywhere, and by definition of the OT, $s(x)=\exp_x[0]=x$.

\subsection{Learning}
We now discuss how to train the proposed flow model. We denote by $\nu_\theta = s_\# \mu$, the prior density pushed by our RCPM model $s$, with parameters $\theta$. To learn a target distribution $\nu$ we consider either minimizing the KL divergence between the generated distribution $\nu_\theta$ and the data distribution $\nu$:
\begin{align}
    \text{KL}(\nu_\theta | \nu) = \E_{\nu_\theta(x)} \big[\log \nu_\theta(x)- \log \nu(x)\big ]\label{e:klloss}
\end{align}
or, minimizing the negative log-likelihood under the model:
\begin{align}
\text{NLL}(\theta) = -\E_{\nu(x)}\log\mu_\theta(x).\label{e:maxlik}
\end{align}
We optimize these with gradient-based methods.

For low-dimensional manifolds, the Jacobian log-determinants appearing
in the computation of KL/likelihood losses can be exactly computed
efficiently. For higher-dimensional manifolds, stochastic trace
estimation techniques can be leveraged \cite{huang2020convex}.

Depending on the considered application, it may be more practical to
parameterize either the forward mapping (from base to target), or the
backward mapping (from target to base). For instance, in the density
estimation context, the backward map from target samples to base
samples is typically parameterized, and can be trained by maximum
likelihood (minimizing NNL) using \cref{e:maxlik}.

\subsection{Smoothing via the soft-min operation}
\label{sec:smooth}
While the proposed layers $s_i$ are universal, they are defined using
the gradients of the discrete $c$-concave potentials that take the
form $\nabla_x \phi(x) = -\log_x(y_i)$, where $y_i$ is the argument
minimizing the r.h.s.~in \cref{e:inf_aff} (see also
\cref{e:grad_pot}). This means that the $\alpha_i$ do
not transfer gradients. Intuitively, considering
\cref{fig:discrete_c_concave}, the $\alpha_i$ represent the heights of
the different $c$-concave pieces and since we only work with their
derivatives, the heights are not ``seen'' by the
optimizer. Furthermore, potential gradients $\nabla_x \phi$ are
discontinuous at meeting points of $c$-concave pieces.

We alleviate both problems by replacing the $\min$ operation by a soft-min operation, $\min_\gamma$, similarly to \citet{softdtw}. The soft-min operation $\min_\gamma$ is defined as
\begin{align}\label{e:min_gamma}
    \min\nolimits_\gamma(a_1,\ldots, a_n) = -\gamma \log \sum_{i=1}^n \exp\parr{-\frac{a_i}{\gamma}}.
\end{align}
In the limit $\gamma \rightarrow 0$,  $\min_\gamma \rightarrow \min$. While $c$-concavity is not guaranteed to be preserved by this modification, it is recovered in the $\gamma \rightarrow 0$ limit. Also, gradients with respect to offsets are not zero anymore, and $\alpha_i^{(j)}$ are optimized through the training process.

\subsection{Discussion}
We now discuss some practice-theory gaps, and mark interesting open questions and future work directions. 

\paragraph{Model smoothing and optimization.} The construction in Section $4$, \ie, exponential map of a discrete $c$-concave function, is an optimal transport map and universal in the sense that it can approximate any OT between an absolutely continuous $\mu$ and arbitrary $\nu$, over a compact Riemannian manifold. It is not, however, a diffeomorphism. As a practical way of optimizing this model to approximate arbitrary $\nu$ we suggested smoothing the min operation with soft-min. If this, now a differentiable function, is $c$-concave then the smoothed version leads to a diffeomorphism (flow). While we are unable to prove that the soft version is $c$-concave, we verified numerically that it indeed leads to a diffeomorphism (see \cref{fig:diffeomorphism}, Appendix). We leave the question of whether the soft-min operator preserves $c$-convexity on the sphere and more general manifolds to future work. Furthermore, it is worth noting that the proposed model can potentially be optimized with other methods than as a flow, for instance by directly optimizing the Wasserstein loss similarly to \citet{makkuva2019optimal} in the Euclidean case, or semi-discrete transport methods \citep{compoptpeyre}. Both would not require the map to be a diffeomorphism. We leave such directions to future work as well.

\paragraph{Scalability.} We follow
\citet{rezende2020normalizing}, which relies on reformulating the
log-determinant in terms of the Jacobian expressed via an orthonormal
basis of the tangent space. The Jacobian determinant term is similar
to the Euclidean case suggesting that the high dimensional case can
reuse techniques from \citet{huang2020convex}.
\Cref{tab:rezende_runtime} shows our running times are comparable to
\citet{rezende2020normalizing}.

\section{Experiments}
This section empirically demonstrates the
practicality and flexibility of RCPMs.
We consider synthetic manifold learning tasks
similar to \citet{rezende2020normalizing,lou2020neural} on
both spheres and tori, and a real-life application over the sphere.
We cover the different use cases of RCPMs: density estimation,
mapping estimation and geodesic transport.

\begin{table}
  \centering
  \caption{We trained a RCPM to optimize the KL 
    on the 4-mode dataset shown in \cref{fig:sphere-data}
    and compare the KL and ESS to the M\"obius-spline flow (MS)
    and exponential-map sum-of-radial flow (EMSRE)
    from \citet{rezende2020normalizing}.
    We report the mean and standard derivation from
    10 trials of the RCPM.
    \vspace{2mm}
  }
  \begin{tabular}{lll} \toprule
    Model & KL [nats] & ESS \\ \midrule
    M\"obius-spline Flow & 0.05 {\small (0.01)} & 90\% \\
    Radial Flow & 0.10 {\small (0.10)} & 85\% \\
    \textbf{RCPM} & \textbf{0.003 {\small(0.0004)}} & \textbf{99.3\%} \\ \bottomrule
  \end{tabular}
  \label{tab:rezende_kl}
\end{table}

\begin{table}
\centering
\caption{Comparison of the runtime per training iteration of our
model with Rezende et al.  over 1000 trials with
batch size of 256.\vspace{2mm}}
\begin{tabular}[t!]{ll}\toprule
Method & Runtime (sec/iteration) \\ \midrule
Radial ($N_T=1$, $K=12$) & $2.05\cdot 10^{-3} \pm 1.33\cdot 10^{-4}$ \\
Radial ($N_T=6$, $K=5$) &  $6.26\cdot 10^{-3} \pm 2.95\cdot 10^{-4}$ \\
Radial ($N_T=24$, $K=1$) & $1.92\cdot 10^{-2} \pm 5.24\cdot 10^{-4}$ \\
RCPM ($N_T=5$, $K=68$) & $8.79\cdot 10^{-3} \pm 1.81\cdot 10^{-4}$ \\ \bottomrule
\end{tabular}
\label{tab:rezende_runtime}
\end{table}

\subsection{Synthetic Sphere Experiments}

\paragraph{KL training.}
Our first experiment is taken from \citet{rezende2020normalizing}, the
task is to train a Riemannian flow generating a $4$-modal distribution
defined on the $\gS^2$ sphere via a reverse-mode KL minimization. This
experiment allows quantitative comparison of the different models'
theoretical and practical expressiveness. We report results obtained
with their best performing models: a M\"obius-spline flow and a
radial flow. The latter is an exponential-map flow with radial layers
($24$ block of $1$ component). We train a $5$-block RCPM.
The exponential map and the intrinsic distance required for RCPMs are
closed-form for the sphere.
More implementation details are provided
in \cref{sec:additional-exps}.

\begin{figure}[t]
  \centering
  \begin{tabular}{cc}
   True & RCPMs \\ \midrule[1pt]
    \includegraphics[width=0.49\columnwidth]{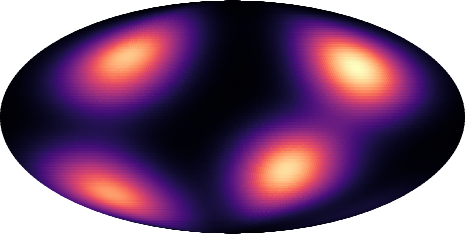} &
    \includegraphics[width=0.49\columnwidth]{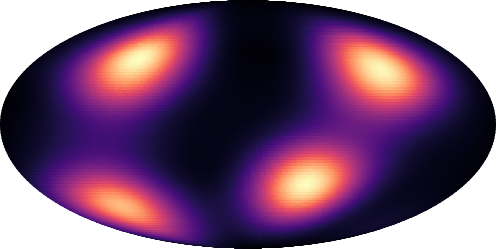} \\
    \includegraphics[width=0.49\columnwidth]{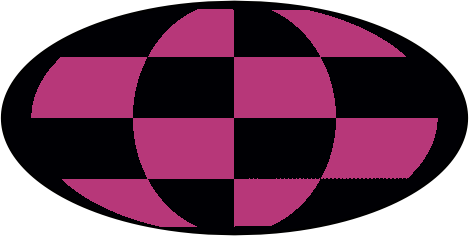} &
    \includegraphics[width=0.49\columnwidth]{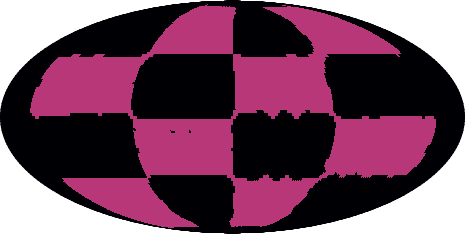} \\
  \end{tabular}
  \caption{
    Learned RCPMs on the sphere.
    Following \citet{rezende2020normalizing},
    we learn the first density with the reverse KL, and
    following \citet{lou2020neural},
    we learn the second with maximum likelihood.
  }
  \label{fig:sphere-data}
\end{figure}

Results are logged in \cref{tab:rezende_kl}. Notably, our model
significantly outperforms both baseline models with a KL of $0.003$,
almost an order of magnitude smaller than the runner-up with a KL of
$0.05$. This highlights the expressive power of the RCPM model
class. We also provide a visualization of the trained RCPM in
\cref{fig:sphere-data} (top row), where we show KDE estimates performed
in spherical coordinates with a bandwidth of 0.2. Finally,
\cref{tab:rezende_runtime} compares the runtime per training iteration of our
model with \citet{rezende2020normalizing}’s models over 1000 trials with
a batch size of 256. Our model’s speed is comparable to
theirs while leading to significantly improved KL/ESS.

\paragraph{Likelihood training.} We demonstrate an RCPM trained
via maximum likelihood on a more challenging dataset, the
checkerboard, also studied in \citet{lou2020neural}.
\Cref{fig:sphere-data} (bottom) shows the RCPM
generated density on the right.
We found
visualizing the density of our model challenging because some regions
had unusually high density values around the poles. We hence binarized
the density plot. We provide the original density in
\cref{fig:binarized-density} (Supplementary).

\subsection{Torus}

\begin{figure*}[t]
    \centering
    \begin{tabular}{ccc}
      Base & RCPM & Target \\ \midrule[1pt]
    \includegraphics[width=0.32\textwidth]{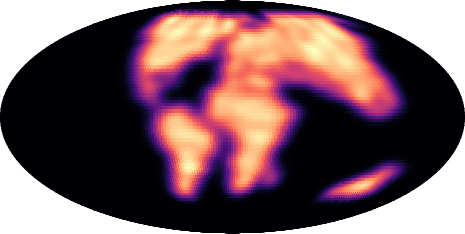} &
    \includegraphics[width=0.32\textwidth]{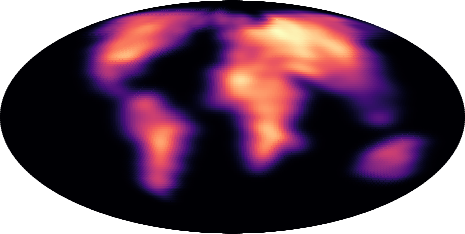} &
    \includegraphics[width=0.32\textwidth]{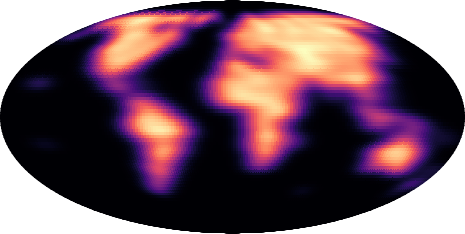} \\
    \end{tabular}
    \caption{We trained a 7-block RCPM flow to learn to map a base density over ground mass on earth of $90$ million years ago such a density over current earth. To learn, we minimize the KL divergence between the model and the target distribution. }
    \label{fig:earth}
\end{figure*}

\begin{figure}[t]
  \begin{tabular}{cc}
   True & RCPM \\ \midrule[1pt]
  \includegraphics[width=0.49\hsize]{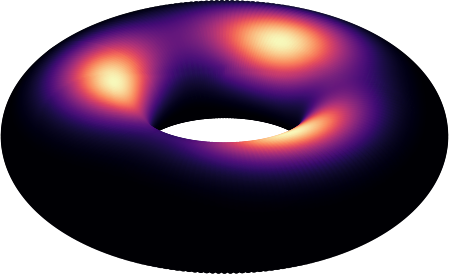} &
  \includegraphics[width=0.49\hsize]{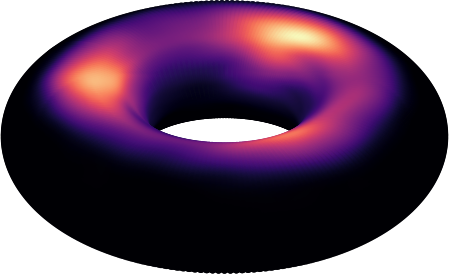} \\
  \end{tabular}
  \caption{We trained an RCPM $\nu_\theta$ to learn a $3$-modal density $\nu$ on the torus $T^2 = \gS^1 \times \gS^1$ (KL: $0.03$, ESS: $94.7$). }
  \label{fig:torus}
\end{figure}
We now consider an experiment on the torus: $\gT^2 = \gS^1 \times
\gS^1$. The exponential map and intrinsic distance
required for RCPMs are known in closed-form.
The exponential map flows in
\citet{rezende2020normalizing} do not apply to this setting as their
$c$-concave layers are specific to the sphere.

We train a 6-block RCPM model (with 1 layer per block) by KL minimization. As can be inspected from \cref{fig:torus}, the RCPM model is able to recover the target density accurately, and the model achieves a KL of $0.03$ and an ESS of $94.7$.

\subsection{Case Study: Continental Drift\footnote{
The source maps of
\cref{fig:earth,fig:geodesicearth,fig:dens_est_earth}
are © 2020 Colorado Plateau Geosystems Inc.}}
Finally, we consider a real-world application of our model on
geological data in the context of continental drift
\citep{contdrift}. We aim to demonstrate the versatility and
flexibility of the framework with three distinct settings: mapping
estimation, density estimation, and geodesic transport, all through
the lens of RCPMs.

\paragraph{Mapping estimation.} We begin with mapping estimation. We aim to learn a flow $t$ mapping the base distribution of ground mass on earth $90$ million years ago (\cref{fig:earth}, left), to a ground mass distribution on current earth (\cref{fig:earth}, right) -- the target. We train a $7$-blocks RCPM with $3$-layers blocks (see \cref{sec:multilayer}) by minimizing the KL divergence between the model and target distributions.
In \cref{fig:earth} (Middle), we show the RCPM result, where it successfully learns to recover the target density over current earth. Hence, the mapping $t$ can be used to map mass from ``old'' earth to current earth.

\paragraph{Transport geodesics.} We demonstrate the use of transport geodesics induced by exponential-map flows. We train a $1$-block RCPM which allows to recover approximations of optimal-transport geodesics following \cref{e:otgeodesics}. These curves are induced by transport mappings $\exp(t\nabla\phi), t\in [0,1]$, which we visualize for a grid of starting points $x_0$ on the sphere in \cref{fig:geodesicearth}. Such geodesics illustrate the optimal transport evolution of earth ground across times. This relates to the well-known and studied geological process of continental drift \citep{contdrift}. North-American and Eurasian tectonic plates move away from each other at a small rate per year, which is illustrated in \cref{e:otgeodesics}. Denote as ``junction'' the junction between  Eurasian and North-American continents in ``old'' earth. We observe that particles $x_0$ located at the right of the junction will have geodesics transporting them towards the right, while particles located at the left of such junction will be transported towards the left, which is the expected behavior given the evolution of continental locations across time (see \cref{fig:earth} left and right). 

\paragraph{Density estimation.} Finally, we consider RCPMs as density estimation tools. In this setting, we aim to learn a flow from a known base distribution (\eg, uniform on the sphere) to a target distribution (\eg, distribution of mass over earth) given samples from the latter. We train an RCPM model with $6$ blocks (and $1$ layer per block) by maximum likelihood. We show the results for this experiment in \cref{fig:dens_est_earth}. We observe that the model is able to recover the distribution of mass on current earth.

\begin{figure}[t]
    \centering
    \includegraphics[width=0.8\hsize]{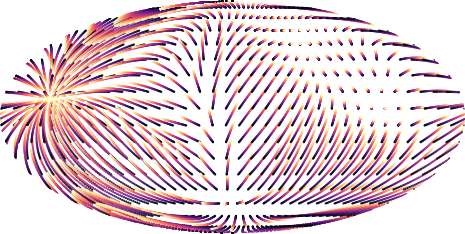}
    \caption{Plot of the transport geodesics arising from a 1-block RCPM trained in the setting of \cref{fig:earth}, and following \cref{e:otgeodesics}. We observe that samples stretch according to continental movements. }
    \label{fig:geodesicearth}
\end{figure}

\begin{figure}[t]
    \centering
    \begin{tabular}{cc}
    True & RCPM \\ \midrule[1pt]
    \includegraphics[width=0.49\columnwidth]{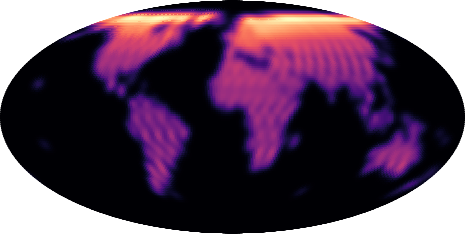} &
    \includegraphics[width=0.49\columnwidth]{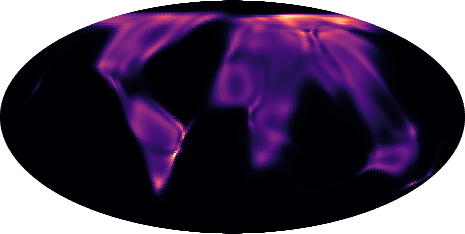}
    \end{tabular}
    \caption{We trained a 6-block RCPM in the density estimation setting. The base distribution is the uniform distribution on the sphere and the target $\nu$ is the ground of current earth.}
    \label{fig:dens_est_earth}
\end{figure}

\section{Conclusion}
In this paper, we propose to build flows on compact Riemannian manifolds following the celebrated theory of \citet{mccann2001polar}, that is using exponential map applied to gradients of $c$-concave functions. Our main contribution is observing that the rather intricate space of $c$-concave functions over arbitrary compact Riemannian manifold can be approximated with discrete $c$-concave functions. We provide a theoretical result showing that discrete $c$-concave functions are dense in the space of $c$-concave functions. We use this theoretical result to prove that maps defined via a discrete $c$-concave potentials are universal. Namely, can approximate arbitrary optimal transports between an absolutely continuous source distribution and arbitrary target distribution on the manifold.  

We build upon this theory to design a practical model, RCPM, that can be applied to any manifold where the exponential map and the intrinsic distance are known, and enjoys maximal expressive power. We experimented with RCPM, and used it to train flows on spheres and tori, on both synthetic and real data. We observed that RCPM outperforms previous approaches on standard manifold flow tasks. We also provided a case study demonstrating the potential of RCPMs for applications in geology. 

Future work includes training flows on more general manifolds, \eg, manifolds defined with signed distance functions, and using RCPM on other manifold learning tasks where the expressive power of RCPM can potentially make a difference. One particular interesting venue is generalizing the estimation of barycenters (means) of probability measures on Euclidean spaces to the Riemannian setting through the use of discrete $c$-concave functions.

\section*{Acknowledgments}
We thank Ricky Chen, Laurent Dinh, Maximilian Nickel and Marc Deisenroth
for insightful discussions
and acknowledge the Python community
\citep{van1995python,oliphant2007python}
for creating
the core tools that enabled our work, including
JAX \citep{jax2018github},
Hydra \citep{Yadan2019Hydra},
Jupyter \citep{kluyver2016jupyter},
Matplotlib \citep{hunter2007matplotlib},
numpy \citep{oliphant2006guide,van2011numpy},
pandas \citep{mckinney2012python}, and
SciPy \citep{jones2014scipy}.

{\small
\bibliography{flow}
\bibliographystyle{icml2021}
}

\clearpage
\appendix
\twocolumn[
\icmltitle{Riemannian Convex Potential Maps: Supplementary Material}
]

\section{Manifold Operations}

We briefly describe manifold operations, on a Riemannian manifold $\gM$ with metric $g$, used in this paper. Specifically, we define the exponential map $\exp$ and the intrinsic manifold distance $d_{\gM}$.

\paragraph{Exponential map.}
Let $x \in \gM$, $v \in T_x\gM$ and consider the unique geodesic $\gamma:[0,1]\rightarrow\gM$ such that $\gamma(0) = x$ and $\gamma'(0) = v$. The exponential map at $x$, $\exp_x : T_x\gM \rightarrow \gM$,
 is defined as 
\begin{align}
    \exp_x(v) = \gamma(1).
\end{align}

\paragraph{Intrinsic distance.} Define the length of a curve $\gamma : [0,1]\rightarrow \gM$ as 
\begin{align}
    L(\gamma) = \int_0^1 \norm{\gamma'(t)}_{g}dt, 
\end{align}
where $\norm{\gamma'(t)}_g$ means taking the norm of the velocity $\gamma'(t)$ at $T_{\gamma(t)} \gM$ with respect to the metric $g$ of the manifold $\gM$. Then, the intrinsic distance $d_{\gM}$ between $x,y \in \gM$ is:
\begin{align}
    d_{\gM}(x,y) = \inf_{\gamma}L(\gamma)
\end{align}
where the $\inf$ is over curves $\gamma:[0,1]\rightarrow\gM$ where  $\gamma(0)=x$ and $\gamma(1)=y$. If $\gM$ is \emph{complete} (see \eg, Hopf-Rinow Theorem) the intrinsic distance is realized by a geodesic.

\paragraph{Sphere.} On the $n$-sphere $\gS^n$, the exponential map and the intrinsic distance are provided as closed-form expressions. If $x, y \in \gS^n$ and $v \in T_x\gS^n$,
\begin{align}
    \exp_x(v) = x \cos(\norm{v}) + \frac{v}{\norm{v}}\sin(\norm{v}) \label{e:expsph}
\end{align}
\begin{align}
    d_{\gS^n}(x,y) = \arccos(x^Ty)\label{e:distsph},
\end{align}
where $\norm{\cdot}$ is the standard Euclidean norm.

\paragraph{Product manifolds.} We now consider operations on product manifolds of the form $\gM = \gM_1 \times\ldots\times \gM_l$. The squared intrinsic distance is simply
\begin{align}
    d^2_{\gM}(x,y) =  d^2_{\gM_1}(x_1,y_1)+\ldots+d^2_{\gM_l}(x_l,y_l).\label{e:decomp_product_cost}
\end{align}
Here  $x = (x_1,\ldots,x_l)$, and $x_j \in \gM_j,\ \ j\in [l]$ (and similarly for $y$). The exponential map on the product manifold is the cartesian product of exponential maps on the individual manifolds. An instantiation of such product that will be considered in experiments is the torus $\gS^1 \times \gS^1$. In that case, we can use \cref{e:expsph,e:distsph} to get the exponential map and squared intrinsic distance in closed-form. 

\section{Proof of c-concavity of the multi-layer potential}

\begin{proof}
The proof is by induction. Constant functions are c-concave, hence $\psi_0$ is c-concave.  Also, $\psi_1=(1-w_0)\phi_0$ is c-concave by the assumption of  convexity of the space of $c$-concave functions.   Next, assuming $\psi_{k-1}(x)$ is c-concave, $\sigma(\psi_{k-1})$ is also c-concave (because $\sigma$ preserves $c$-concavity), and $\psi_k(x)$ is c-concave because convex combinations of $c$-concave functions are $c$-concave. In conclusion, $\varphi = \psi_K$ is $c$-concave
\end{proof}

\section{Additional experimental and implementation details}
\label{sec:additional-exps}

\subsection{Synthetic Sphere}
We conducted a hyper-parameter search over the parameters in
\cref{table:sweep} to find the flows used in our demonstrations
and experiments. We report results from the best hyper-parameters
obtained by randomly sampling the space of parameters.
The $\alpha$ values are initialized from
$\gU[\alpha_{\rm min}, \alpha_{\rm min}+\alpha_{\rm range}]$. Also, $\gamma_1$ corresponds to the softing coefficient of the soft-min operation of discrete $c$-concave potentials, and $\gamma_2$ to the softing coefficient of the soft-min operation in the identity initialization (see \cref{sec:multilayer,sec:smooth}).

\begin{table}[!ht]
  \centering
  \caption{Hyper-parameter sweep for our sphere results}
  \label{table:sweep}
  \begin{tabular}{ll} \toprule
    \multicolumn{2}{c}{Adam} \\ \midrule
    learning rate & [$10^{-6}$, $10^{-1}$] \\
    $\beta_1$ & [0.1, 0.3, 0.5, 0.7, 0.9] \\
    $\beta_2$ & [0.1, 0.3, 0.5, 0.7, 0.9, 0.99, 0.999] \\
    \midrule \multicolumn{2}{c}{Flow Hyper-parameters} \\ \midrule
    Nb. of Components $y_i$ & [50, 1000] \\
    $\alpha_{\rm min}$ & [$10^{-5}$, 10] \\
    $\alpha_{\rm range}$ & [$10^{-3}$, 1] \\
     $\gamma_1$ & [0.01, 0.05, 0.1, 0.5] \\
     $\gamma_2$ & [None, 0.01, 0.05, 0.1, 0.5] \\
  \end{tabular}
\end{table}

We now verify empirically whether the RPCM define diffeomorphisms in practice. We compute Jacobian log-determinants of the flow trained on the $4$-modal density taken from \citet{rezende2020normalizing} for $10^6$ points uniformly sampled on the sphere, and observe that all these are positive (see \cref{fig:diffeomorphism}).
\begin{figure}
    \centering
    \includegraphics[width=0.5\hsize]{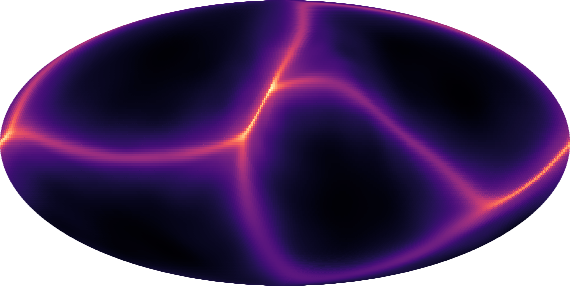}
    \caption{Jacobian log-determinants for points uniformly sampled on the sphere.}
    \label{fig:diffeomorphism}
\end{figure}

\paragraph{Binarized checkerboard density.}
We found it difficult to visualize the learned
density of our model on the checkerboard because a few
regions have unusually high values that mess up
the ranges of the colormap.
For visualization purposes, we binarize the density
values by taking the portion of the density greater
than the uniform density.
\Cref{fig:binarized-density} shows the original
and binarized densities of our models.

\begin{figure}
  \newcommand{\alphamin}{\ensuremath{\alpha_\mathrm{min}}}
  \newcommand{\alphamax}{\ensuremath{\alpha_\mathrm{max}}}
  \begin{tabular}{cc}
   Original Density & Binarized Density \\
  \includegraphics[width=0.45\columnwidth]{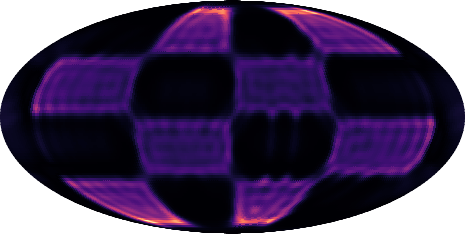} &
  \includegraphics[width=0.45\columnwidth]{images/checkerboard-binary.png}
  \end{tabular}
  \caption{Binarized density of the sphere checkerboard}
  \label{fig:binarized-density}
\end{figure}

\subsection{Torus}
\paragraph{Model.}
We provide details on the model used in the torus demonstration. The RCPM is composed of $6$ single-layer blocks of $200$ components, and the softing parameter is set to $0.5$. Adam's learning rate is set to $6e^{-4}$ and $\beta$ to $(0.9,0.99)$.

\paragraph{Data.}
The target density is of a form inspired by the target densities in \citet{rezende2020normalizing}):
\begin{align}
    p(\theta_1,\theta_2) &= \frac{1}{3}\sum_{i=1}^3p_i(\theta_1, \theta_2)\\
    p_i(\theta_1,\theta_2)&\propto \exp{[\cos(\theta_1 - a^1_i)+\cos(\theta_2 - a^2_i)]}
\end{align}
where $a_1 = [4.18, 6.7], a_2 = [4.18, 4.7], a_3 = [4.18, 2.7]$, and $\theta_1, \theta_2 \in [0,2\pi]$.

\subsection{Continental Drift}
\paragraph{Mapping estimation.}
We continue with details on the model used in the mapping estimation setting of the continental drift case study. The RCPM is composed of $7$ blocks containing each $3$ layers with $200$ components, and the softing coefficient is set to $0.2$. Adam's learning rate is set to $2e^{-3}$ and $\beta =(0.9,0.99)$.

\paragraph{Transport geodesics.}
We now discuss the transport geodesics setting. The RCPM is composed of a single block (hence allowing to recover the optimal transport geodesics) containing $3$ layers with $200$ components, and the softing coefficient is set to $\gamma = 0.2$. Adam's learning rate is set to $2e^{-3}$ and $\beta =(0.9,0.99)$.

\paragraph{Density estimation.}
Finally, we provide details on the model used in the density estimation setting. The RCPM is composed of $6$ single-layer blocks containing each $400$ components, and the softing coefficient is set to $6e^{-2}$. Adam's learning rate is set to $2e^{-3}$ and  $\beta =(0.9,0.99)$.

\paragraph{Data.} The earth densities are obtained by leveraging the code from \url{https://github.com/cgarciae/point-cloud-mnist-2D} to turn Mollweide earth images into spherical point clouds, converting to Euclidean coordinates, and applying kernel density estimation to such point clouds both for visualization, and to get log-probabilities when they are required (e.g., in the mapping estimation setting, where access to log-probabilities from the base -- old earth -- is needed to train the model).

\end{document}